\documentclass{article}

\usepackage[utf8]{inputenc} 
\usepackage[T1]{fontenc}    
\usepackage{url}            
\usepackage{booktabs}       
\usepackage{amsfonts}       
\usepackage{nicefrac}       
\usepackage{microtype}      

\usepackage{amsmath,amssymb} 
\usepackage{pgfplots}
\usepackage{subfigure}
\usepackage{tikz}

\newcommand{\reals}{\mathbb{R}}

\newcommand{\lemref}[1]{Lemma~\ref{#1}}
\newcommand{\thmref}[1]{Theorem~\ref{#1}}
\newcommand{\plh}{{\mkern-1.5mu\times\mkern-1.5mu}}

\DeclareMathOperator*{\E}{\mathbb{E}}

\newtheorem{lemma}{Lemma}

\newtheorem{claim}{Claim}
\newtheorem{theorem}{Theorem}

\newtheorem{definition}{Definition}

\newcommand{\BlackBox}{\rule{1.5ex}{1.5ex}}  
\newenvironment{proof}{\par\noindent{\bf Proof\ }}{\hfill\BlackBox\\[2mm]}

\begin{document}

\newcommand{\todo}[1]{\textcolor{red}{TODO: #1}\PackageWarning{TODO:}{#1!}}

\title{Learning a Metric Embedding \\
  for Face Recognition\\
  using the Multibatch Method}

\author{
Oren Tadmor \and
Yonatan Wexler\and
Tal Rosenwein \and
Shai Shalev-Shwartz\and
Amnon Shashua
}
\date{{\small\texttt{firstname.lastname@orcam.com}}\\
  \vspace{5pt}
  Orcam Ltd., Jerusalem, Israel}
\maketitle

\begin{abstract}
  This work is motivated by the engineering task of achieving a near
  state-of-the-art face recognition on a minimal computing budget
  running on an embedded system.  Our main technical contribution
  centers around a novel training method, called Multibatch, for
  similarity learning, i.e., for the task of generating an invariant
  ``face signature'' through training pairs of ``same'' and
  ``not-same'' face images. The Multibatch method first generates
  signatures for a mini-batch of $k$ face images and then constructs
  an unbiased estimate of the full gradient by relying on all $k^2-k$
  pairs from the mini-batch. We prove that the variance of the
  Multibatch estimator is bounded by $O(1/k^2)$, under some mild
  conditions. In contrast, the standard gradient estimator that relies
  on random $k/2$ pairs has a variance of order $1/k$. The smaller
  variance of the Multibatch estimator significantly speeds up the
  convergence rate of stochastic gradient descent.  Using the
  Multibatch method we train a deep convolutional neural network that
  achieves an accuracy of $98.2\%$ on the LFW benchmark, while its
  prediction runtime takes only $30$msec on a single ARM Cortex A9
  core. Furthermore, the entire training process took only 12 hours on
  a single Titan X GPU.
\end{abstract}

\section{Introduction}

Face representation through a deep network embedding is considered the
state of the art method for face verification, face clustering, and
recognition. The deep network is responsible for mapping the raw image
of a detected face, typically after an alignment phase for pose
correction, into a ``signature'' vector such that signatures of
two pictures of the same person have a small distance while
signatures of images of different individuals have a large
distance. 

The various approaches for deep network embeddings differ along three
major attributes: first, and foremost, by the way the loss function is
structured. Typically, the loss can be represented by the expectation
of all pairs of ``same'' and ``not same'' images, or more recently as
proposed by~\cite{Schroff2015FaceNet} as an expectation over triplets
of images consisting of two ``same'' and a ``not same'' image. The second
attribute is the network architecture: whether or not it assumes an
alignment pre-process, whether it has locally connected layers
\cite{Taigman2014deepface}, whether it is structured as a conventional
concatenation of convolution and pooling
\cite{parkhi15deep,Zeiler2015} or subsumption architecture
\cite{Schroff2015FaceNet}. The third attribute has to do with the use
of a classification layer~\cite{Taigman2014deepface} trained
over a set of known face identities. The signature vector is then taken
from an intermediate layer of the network and used to generalize
recognition beyond the set of identities used in training.

Our goal in this work is to design a deep network embedding for face
representation that is highly optimized in run-time performance, has a
very high precision --- close to the best, yet very expensive,
networks and can be trained overnight using reasonable computing
resources while using a training set of a manageable
size. Specifically, using a training set of 2.6M face $112\plh 112$
images, running overnight on a single Titan X GPU card, our network
has achieved an accuracy of $98.2\%$ on the LFW benchmark
(significantly higher than the best 2014 performance of
\cite{Taigman2014deepface} at 1/5000 of the running time) and it runs
end-to-end using $41M$ FLOPs. This amounts to a runtime of $30$ms on a
single ARM Cortex A9 core (i.mx6 solo-light microprocessor by NXP,
running at 1 GHz).

We train our network using the objective that requires the signatures
of all ``same'' pairs to be below a global threshold while all the
signatures of ``not same'' pairs should be above the global
threshold. In Section~\ref{sec:hard} we show that fulfilling this
objective is {\sl harder\/} than training a multiclass categorization
network over a set of known identities (the latter being the method
used by~\cite{Taigman2014deepface} as a surrogate objective). While
solving a multiclass categorization is an ``easier'' task, it requires
a larger training set because of the significant increase in the
number of parameters due to the large output layer (consisting of
thousands of nodes, one per face identity).

Our main technical contribution centers around a better mechanism for
constructing gradient estimates for the SGD method. The estimator,
called Multibatch, is described in Section~\ref{sec:main}. We prove
that the variance of the Multibatch estimator decreases (under some
mild conditions) as $O(1/k^2)$. In contrast, the variance of the
standard gradient estimate is order of $1/k$. The smaller variance of
the Multibatch estimator, which comes with nearly no extra
computational cost per iteration, is translated to a faster
convergence of SGD, both in theory and in practice. This dramatic
speed-up enables us solve the hard problem of metric embedding without
relying on surrogates such as the multiclass surrogate of
\cite{Taigman2014deepface} or the triplet surrogate of \cite{Schroff2015FaceNet}.

From the engineering standpoint, we propose a streamlined network
which combines the pre-processing step of face alignment into one
network that can be optimized end-to-end.
Specifically, the alignment pre-process is necessary for reducing the
variability of the problem and thereby both reducing the size of the
required training set and increasing the accuracy of the
classifier. For the sake of unified implementation and optimization of
code on the embedded platform we trained a small convolutional network
with $4.8M$ FLOPs to produce the warping parameters as given by the
state-of-the-art face alignment algorithm of~\cite{kazemi2014}. The
result is a system comprised solely of convolutional and linear
transformations using a pipeline starting from a pre-process deep
network, a warping transformation followed by the deep network
embedding responsible for mapping the warped face image to a signature
vector. The entire pipeline takes $41M$ FLOPs and runs on
$112\plh 112$ face images with a runtime of $30$ms on a single Cortex
A9 core running at $1GHz$.


\section{Learning a Metric}
\label{sec:hard}

\begin{figure}[t!]
  \centering
  \includegraphics[clip, trim=0.5cm 6.5cm 0.5cm 6.5cm, width=0.8\textwidth]{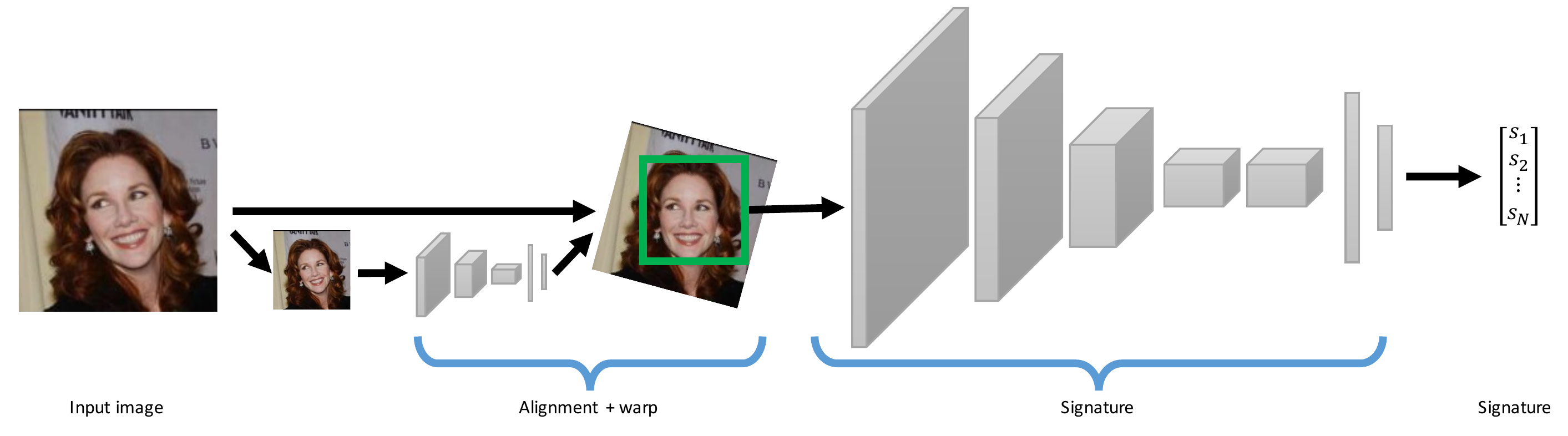}
  \caption{Overview of joint alignment and signature generation
    network}
  \label{fig:faces-pipe}
\end{figure}

We seek to learn a transformation $f_w$, parameterized by $w$, from
input space into $\mathbb R^d$.  The resulting vectors should agree
with ``same-not-same'' criterion where two samples from the same class
should be close up to some learned threshold $\theta$ and vectors from
differing classes should be farther than $\theta$ under a norm of
choice.  Namely, for two samples $x,x'$ and their labels
$y,y'\in [k]$:
\begin{eqnarray}
  \label{eq:SNS}
  y=y'      & \ \ \ \Longrightarrow\ \ \  & \|f_w(x) - f_w(x') \|^2 <\theta -1\nonumber\\
  y\neq y'  & \ \ \ \Longrightarrow\ \ \  & \| f_w(x) - f_w(x') \|^2 >\theta +1
\end{eqnarray}
The definition above satisfies our needs as it finds a metric
embedding and a single global threshold.  It enables any open-dictionary
scenario, such as the case of face recognition.  The existence of a
global threshold is important as the typical scenario involves
seeing many faces that are not known (e.g.\ people walking down the street).

We would like to optimize the loss over the parameters $w$ of the
network, including the value of the threshold $\theta$.  Given a
training set ${\{(x_i,y_i)\}}_{i=1}^m$ we define the per-pair loss
function:
\begin{equation}
{l}(w,\theta;\ x_i, x_j,y_{ij}) = {\left(1-y_{ij}\left(\theta -
        \|f_w(x_i)-f_w(x_j)\||^2\right)\right)}_+
  \label{eq:pair-loss}
\end{equation}
where $y_{ij}\in\{\pm 1\}$ indicates whether the objects $x_i,x_j$ are of the
same class or not, and $(u)_+:=\mathrm{max}(u,0)$.
The global loss over the whole set is the average loss over all the pairs:
\begin{equation}
  \mathrm{L}(w, \theta) =
  \frac{1}{m^2-m}\sum_{i\neq j \in [m]}{l}(w,\theta;\ x_i, x_j,y_{ij}) ~.
  \label{eq:objective}
\end{equation}
Learning a ``same/not-same'' classifier based on the above approach has
been studied extensively. See for example
\cite{xing2003distance,shalev2004online}.  However, in the context of
face recognition using deep networks, most previous works avoid using
$L(w,\theta)$ directly and instead define an over-subscribed
multiclass problem.  For example \cite{Taigman2014deepface} used a
convolutional neural net to learn both $w$ and a matrix-vector pair
$(W,b)$ such that $W\cdot f_w(x_i) + b$ is a vector whose largest
coordinate is the correct class, $y_i\in [N]$, where $N$ is the number
of different labeled identities in the training set ($N=4030$ was used
in \cite{Taigman2014deepface}). Then, a multiclass loss function is
used, for example the multiclass hinge-loss:
\begin{equation}
  \mathrm{MC}(w, W,b) = \sum_i \, \max_{t \in [N]}
  \left(1[t \neq y_i] + W_t\cdot f_w(x_i) + b_t - W_{y_i} \cdot
    f_w(x_i) - b_{y_i} \right) ~,
  \label{eq:MCLoss}
\end{equation}
where $W_t$ is the $t$'th row of the matrix $W$. Another popular
multiclass loss is the logistic loss (which was used in
\cite{Taigman2014deepface}).  The embedding $f_w(x)$ is the so-called
last ``representation layer'', namely, the layer preceding the output
layer.  In \cite{Taigman2014deepface} the representation layer has
$4096$ nodes.  In other words, the embedding $f_w$ is achieved
indirectly through a multi-class problem with $N$ different labels
with the tradeoff of adding the matrix $W$ with $4096\plh 4030$
parameters, which in fact is the majority of parameters of the entire
network.  The multiclass approach is therefore highly over-subscribed
compared to the metric learning approach.  We show next that indeed
the metric learning approach is ``harder'' than the multiclass
approach, explaining the motivation for learning embeddings in such an
indirect manner.

\begin{figure}
  \centering
  \begin{minipage}[b]{0.40\linewidth}
  \centering
  \includegraphics[trim=0.5cm 2cm 0.5cm 4cm, width=0.8\textwidth]{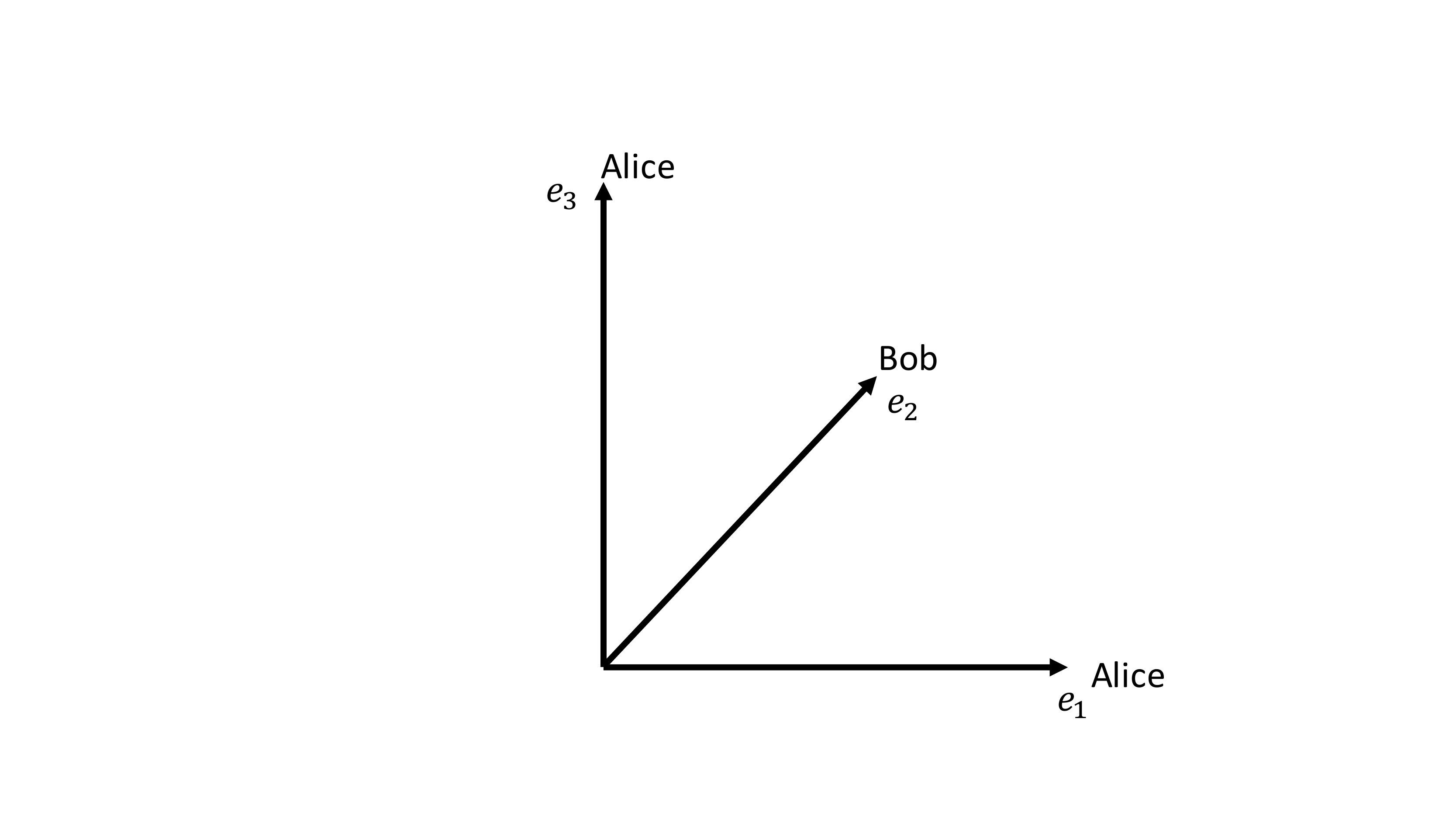}
  \caption{\small Why multiclass is easier than same-not-same: In this
    example, three samples of two identities are mapped to the
    canonical vectors $e_i$ where Alice maps to $e_1$, $e_3$ and Bob
    maps to $e_2$.  It is easy to see that $\|e_1-e_2\|=\|e_1-e_3\|$
    so they are all equidistant.  Still, a multiclass classifier is
    easy with $W_\mathrm{Alice} = e1+e3$ and $W_\mathrm{Bob} = e2$. }
    \label{fig:multiclass-is-harder}

  \end{minipage}
  \hspace{0.05\linewidth}
  \begin{minipage}[b]{0.50\linewidth}
  \centering
  \includegraphics[trim=0.5cm 0.5cm 0.5cm 0.5cm, width=0.6\textwidth]{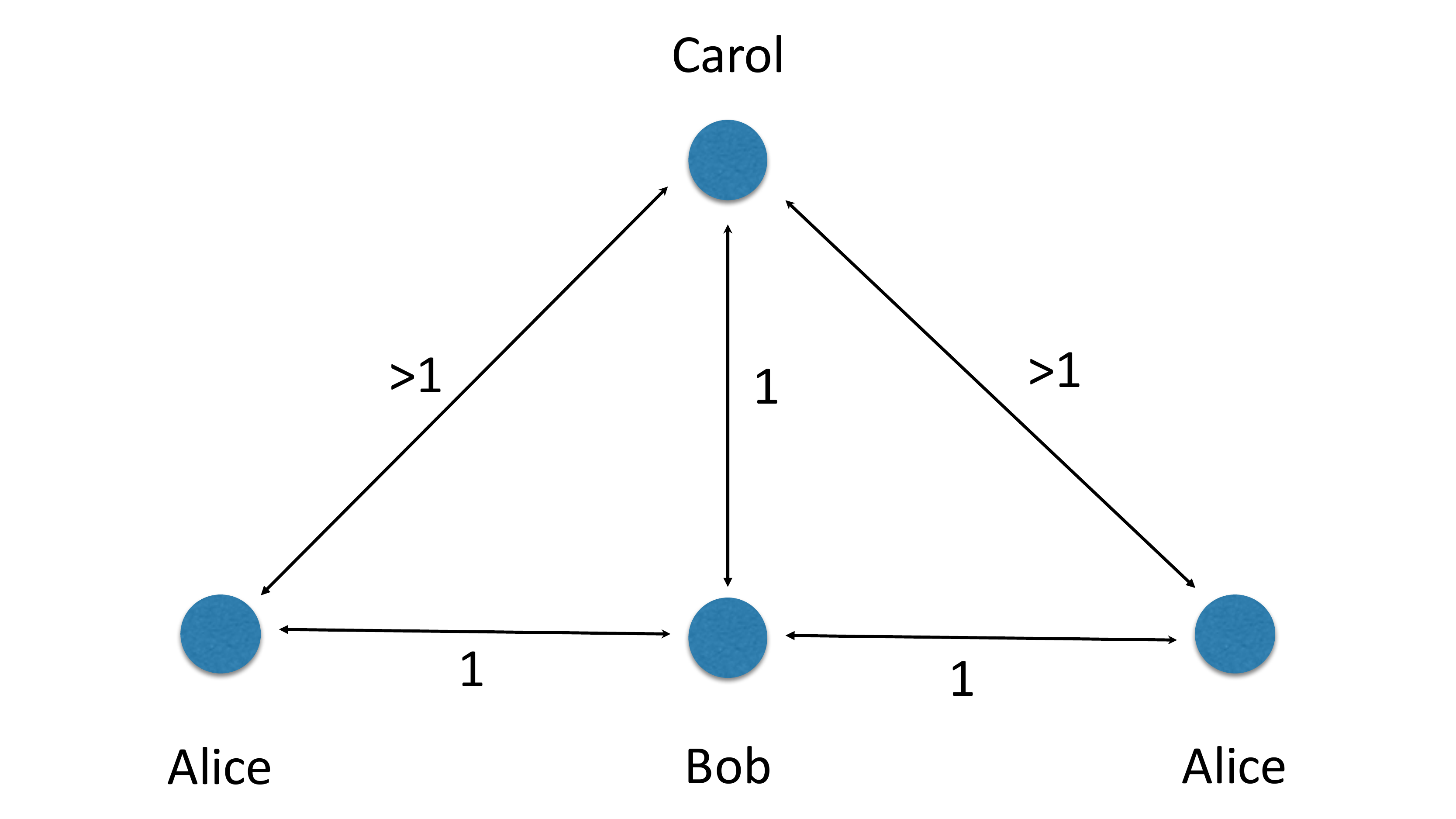}
  \caption{\small An ambiguous configuration.  In this set, choosing
    either Alice-Bob pair will not improve the solution as both Alices
    will try to move closer to each-other while Bob will push them
    apart.  Choosing Alice-Carol will not result an improvement since
    the distance is already larger than the margin.  Only choosing
    Carol-Bob will improve the result.  When choosing only pairs, the
    chance of choosing Carol-Bob is 1:3}
  \label{fig:alice-bob-carol}
  \end{minipage}
\end{figure}

\begin{claim}
Metric learning is harder than multiclass in the
following sense: Given $w,\theta$ for which Eq.~\ref{eq:objective} is
zero, there exists $(W,b)$ for which the value of Eq.~\ref{eq:MCLoss}
at $w,W,b$ is zero.  However, there exists $w,W,b$ such that
Eq.~\ref{eq:MCLoss} is zero but such that for every $\theta$, the
value of Eq.~\ref{eq:objective} at $w,\theta$ is substantially greater
than zero.
\end{claim}
\begin{proof}
Let $w,\theta$ be such that Eq.~\ref{eq:objective} is zero. For every
$k$, let $C_k$ be the set of indices of examples that belong to class
$k$.  Define $W$ to be the matrix such that its $k$'th row is
$W_k = \frac{1}{|C_k|} \sum_{i \in C_k} f_w(x_i)$.  Clearly, for every
$j$
\begin{align*}
&W_k \cdot f_w(x_j) = \frac{1}{|C_k|} \sum_{i \in C_k} f_w(x_i)^\top
f_w(x_j) \\
&= - \frac{1}{2|C_k|} \sum_{i \in C_k} \|f_w(x_i)-f_w(x_j)\|^2
+ \frac{1}{2|C_k|} \sum_{i \in C_k}  \|f_w(x_i)\|^2 + \frac{1}{2} \|f_w(x_j)\|^2
\end{align*}
The last term in the above does not depend on $k$, and therefore it
can be omitted since it does not effect the multiclass loss given in
Eq.~\ref{eq:MCLoss}. Let
$b_k = -\frac{1}{2|C_k|} \sum_{i \in C_k} \|f_w(x_i)\|^2$, and hence
it cancels out the second term. It is left to deal with the first
term. If $j \in C_k$ then $\|f_w(x_i)-f_w(x_j)\|^2 \le \theta - 1$,
hence the first term is at least $-\frac{\theta - 1}{2}$. If
$j \notin C_k$ then $\|f_w(x_i)-f_w(x_j)\|^2 \ge \theta + 1$, hence
the first term is at most $-\frac{\theta+1}{2}$. Plugging into
Eq.~\ref{eq:MCLoss} we obtain that its value is zero.  Finally, to
show that the converse isn't true, consider
Figure~\ref{fig:multiclass-is-harder} which depicts three samples from
two classes which are equidistant but are easily separable. 
\end{proof}

Our experiments show that optimizing Eqn.~\ref{eq:objective} indeed
converges slowly with SGD.\@ While we later provide a mathematical
analysis of this, let us first consider the typical and intuitive
example in Figure~\ref{fig:alice-bob-carol}.  There, four points are
mapped onto the target space, hence providing $6$ pair-wise
constraints.  As is typical, the classes are mixed and a sample of Bob
lies between two samples of Alice.  It is easy to see that among all
pairs, only the Carol-Bob pair is useful.  Neither Alice-Bob pairs
help as it would push Bob to be closer to another Alice (or have small
effect pushing Bob outwards).  An Alice-Carol pair would not result
any loss as they are already far enough.  Only by considering all
pairs of samples, including Carol, that the correct update step can be
reliably estimated.

The work of \cite{Schroff2015FaceNet} used triplets in order to define
an easier goal where pairs from each class should be closer than pairs
from differing classes.  That objective forgoes one global threshold
and so is easier to solve.  In order to achieve convergence they
suggest a boosting scheme where ``hard but not too hard'' examples are
picked using a separate process.  In the next section we show how to
extend this idea to an arbitrary number of pairs and solve the harder
problem faster.

\section{The Multi-Batch Estimator}
\label{sec:main}

In this section we describe the multi-batch approach for minimizing
$L(w,\theta)$ (as given in Eq. \eqref{eq:objective}) and analyze its main
properties. To simplify the presentation, denote $z = (w,\theta)$. Our
goal is therefore to minimize $L(z)$. We also use the notation
$\ell_{i,j}(z)$ to denote $\ell(w,\theta;x_i,x_j,y_{i,j})$.

The multi-batch method relies on the Stochastic Gradient Descent (SGD)
framework, which is currently the most popular approach for training
deep networks. SGD starts with some initial solution, $z^{(0)}$. At
iteration $t$ of the SGD algorithm, it finds an estimate of the
gradient, $\nabla L(z^{(t-1)})$, and updates $z$ based on the
estimation. The simplest update rule is
$z^{(t)} = z^{(t-1)} - \eta \hat{\nabla}^{(t)}$, where
$\eta \in \reals$ is a learning rate and $\hat{\nabla}^{(t)}$ is the
estimate of $\nabla L(z^{(t-1)})$.

As indicated by the term "stochastic", the estimate of the gradient
is a random vector, and we require that the estimate will be unbiased,
namely that $\E[\hat{\nabla}^{(t)} ] = \nabla L(z^{(t-1)})$. In words,
on average, the update direction is the gradient direction.

We focus on a family of gradient estimates, in which the objective is
being rewritten as an average of $n$ functions,
$L(z) = \frac{1}{n} \sum_{i=1}^n L_i(z)$, and the estimated gradient
is obtained by sampling $i$ uniformly at random from $[n]$ and setting
$\hat{\nabla}^{(t)} := \nabla L_i(z)$. Due to the linearity of the
derivation operator, we clearly have that this yields an unbiased
estimate:
$\E[ \hat{\nabla}^{(t)} ] = \E[ \nabla L_i(z)] = \nabla \E[ L_i(z)] =
\nabla L(z)$.

In our metric learning problem, the objective, $L(z)$, is already
defined as an average of $m^2-m$ functions: for every pair
$j,j' \in [m]$, $j \neq j'$, we can define $L_i(z) =
\ell_{j,j'}(z)$. While this approach gives us an unbiased estimate of
the gradient, the \emph{variance} will be order of $1$. In recent
years, the importance of the variance of SGD algorithms has been
extensively studied. Formally, the variance is defined as follows:
\begin{definition}[variance of gradient estimate]
Let $L(z) = \frac{1}{n} \sum_{i=1}^n
L_i(z)$. The variance of the natural unbiased gradient
estimate is defined as:
\[
\nu^2(z) := \frac{1}{n} \sum_{i=1}^n \| \nabla L_i(z) - \nabla L(z) \|^2  ~.
\]
\end{definition}
Several recent works analyzed the convergence of SGD as a function of
the variance. See for example \cite{moulines2011non,
  shalev2013stochastic, johnson2013accelerating, shalev2016sdca,
  zhao2014accelerating, needell2014stochastic}.  Ghadimi and Lan
\cite{ghadimi2013stochastic} have shown that even for non-convex
problems, the number of iterations required by SGD to converge (in the
sense of finding a point for which the squared norm of the gradient is
at most $\epsilon$) is order of
$\frac{\bar{\nu}^2}{\epsilon^2} + \frac{1}{\epsilon}$, where
$\bar{\nu}^2$ is an upper bound on $\nu^2(z^{(t)})$ for every $t$.

To decrease the variance, a common approach in deep learning is to
rely on mini-batches. Formally, for any $L$ of the form
$L(z) = \frac{1}{n} \sum_{i=1}^n L_i(z)$, we can also rewrite $L$ as
$L(z) = \frac{1}{n^k} \sum_{i_1,\ldots,i_k} L_{i_1,\ldots,i_k}(z)$,
where $L_{i_1,\ldots,i_k}(z) = \frac{1}{k} \sum_{r=1}^k
L_{i_r}(z)$. This leads to the gradient estimate based on a
mini-batch: sample $i_1,\ldots,i_k$ i.i.d. from $[n]$ and set the
gradient to be the average value of $\nabla L_{i_j}(z)$, where
averaging is over $j \in [k]$.  Since the random vector is now an
average of $k$ i.i.d. random vectors, the variance decreases by a
factor of $1/k$.

Getting back to our metric learning problem, estimating the gradient
with a mini-batch of size $k/2$ requires the sampling of $k/2$ pairs
of instances and calculating the loss for every pair. We observe that
the main computation time for this operation is in the ``forward'' and
``backward'' calculations performed by the signature network over the
$k$ instances. The time requires to calculate the gradient of the loss
given the signatures is negligible.  This raises a natural question:
if we have a budget of performing a forward and backward pass on $k$
instances, can we find another unbiased estimate of the gradient whose
variance is strictly smaller than order of $1/k$ ? The multibatch
method aims at achieving this goal.

Intuitively, once we have calculated the signatures of $k$ instances,
it would be wasteful not to use the loss on all the $k^2-k$ pairs of
different instances. The multibatch estimate does exactly that --- it
samples $k$ instances and defines the gradient to be the gradient of
the average loss on all $k^2-k$ pairs of different instances.
Formally, let $\Pi$ denote the set of permutations over $[m]$, and for
every $\pi \in \Pi$ we define
\[
L_\pi(z) = \frac{1}{k^2-k} \sum_{i \neq j \in [k]}
\ell_{\pi(i),\pi(j)}(z) ~.
\]
The multibatch estimate samples $\pi$ uniformly at random from $\Pi$
and returns $\nabla L_\pi(z)$ as the gradient estimate.
The following theorem shows that the multibatch estimate is an
unbiased estimate and that the variance of the multibatch estimate can
be order of $1/k^2$ under mild conditions.
\begin{theorem}\label{thm:mainVarianceMultibatch}~
  \begin{itemize}
  \item $\E_\pi [ \nabla L_\pi(z)] = \nabla L(z)$
  \item For every $i,r$, denote
    $\bar{A}^{(r)}_i = \E_{j \neq i} (\nabla_r \ell_{i,j}(z) - \nabla_r
    L(z))$. Then,
    \[
      \E_\pi \| \nabla L_\pi(z) - \nabla L(z)\|^2 \le
      \frac{1}{k^2-k} \E_{i \neq j} \| \nabla \ell_{i,j}(z)
      - \nabla L(z)\|^2 +  O\left( \frac{1}{k} \right) \, \sum_r
      \E_{i \sim [m]} {(\bar{A}^{(r)}_i)}^2 ~.
    \]
  \end{itemize}
\end{theorem}
The proof of the theorem is given in Section~\ref{sec:proofMain}.  The
bound on the variance of the multibatch method is comprised of two
terms. The first term is the variance of the vanilla estimate (based
on a single pair) divided by $k^2 - k$. That is, for the first term we
obtain a decrease of factor order of $k^2$. This is much better than
the decrease of factor $k/2$ if we take $k/2$ pairs. It is easy to
verify that the second term is at most $O(1/k)$ times the variance of
the vanilla estimate. However, in most cases, we expect
$\bar{A}^{(r)}_i$ to be close to zero. Intuitively, if for most of the
$i$'s, the expected value of $\nabla \ell_{i,j}(z)$ over $j \neq i$ is
roughly the expected value of $\nabla \ell_{i,j}(z)$ over all pairs
$i \neq j$, then the second term in the bound will be close to zero.
We note that one can artificially construct $z$ for which the variance
is truly order of $1/k$, but such vectors correspond to bad metrics,
and in practice, as long as the learning process succeeds, we observe
a variance of order $1/k^2$.

\section{Experiments on Face Recognition}
\label{sec:arch}

Face recognition is an ideal test bench for similarity learning as it
incorporates the need to map the input image into an invariant
``signature'' while meeting very high performance bars.  Top scores on
public benchmarks come mostly from large corporations with access to
extensive resources both in the amount of training data and computing
infrastructure. We demonstrate below that with the multibatch training
algorithm we can accomplish near state-of-the-art precision using a
fraction of the training set size, a fraction of the training time and
a fraction of the prediction running time of top performing
networks. We describe below the method for preparing a training set of
$2.6M$ face images of $12k$ different identities, our method for
incorporating face alignment into the process and the architecture of
our convolutional network.

The leading approach for face recognition comprises of two steps.
After detection, the first step is face alignment where the face is
rectified to a fixed position and scale.  This aligned image is then
fed to the second step which computes a signature.  Signatures are
compared using a standard distance function in $\mathbb R^d$ to enable
fast comparison in the open-dictionary.  Figure~\ref{fig:faces-pipe}
shows the network.

With regard to training sets, the available public datasets are not
sufficiently substantial to allow high performance recognition.  We
follow the approach of~\cite{parkhi15deep} and search online for
images of $12k$ public figures that are not included in the LFW
benchmark~\cite{LFWTech}.  As the search results are roughly sorted by
decreasing reliability, we assume the first $50$ of each identity are
correct and train a binary SVM classifier to separate them from a
random set of $10k$ other faces.  This classifier is then used to pick
correct images out of the average of $500$ images per identity.  We
further leverage on~\cite{parkhi15deep} by using their published
network as a feature and skip any manual filtering.  This produced
$2.6$ million images.  By manually verifying $4,000$ images we found
$1.17\%$ errors.  For the sake of this work, we did not clean the set
further.

With regard to alignment, previous work
(cf. \cite{Taigman2014deepface}) underscores the importance of image
rectification pre-processing before being fed into a convolutional
network. In \cite{Taigman2014deepface} 3D rectification was proposed
but more modest 2D rectification has been used by other systems as
well (\cite{parkhi15deep, Schroff2015FaceNet}).  Rectification allows
the model to focus on inter-person difference rather than alignment
hence relieving the classifier from requiring an ensemble of networks
which are expensive at runtime.

We achieve the rectification goal through a 2D similarity
transformation using a compact convolutional network with $102K$
parameteres taking $4.8M$ FLOPs which takes around $5$msec run-time on
a single Cortex A9 core.  The network is trained on images of size
$50\plh 50$ with desired output determined by the algorithm of
\cite{kazemi2014} which uses random forests to detect the location of
$68$ facial landmark points. These points were used to compute warping parameters 
(rotation+translation+scale) which are then fed as
ground-truth to train the ConvNet using a warping layer
per~\cite{JaderbergSZK15Spatial}.  The trained alignment network was
used as the initial guess for the complete network described in
Table~\ref{fig:CNNmodel}.

While the code and model for~\cite{kazemi2014} are available online,
there are two advantages in replacing it with a ConvNet.  First, from
an implementation perspective it is natural to spend resources in
optimizing a ConvNet on an embedded system rather than optimizing
multiple algorithms. Second, the amount of run-time memory required by
random forests is quite high for embedded use.  It stands on $96MB$ in
the available implementation --- which is roughly a factor of $1400$
compared to our ConvNet memory requirements.

Next, with regard to ConvNet architecture, the input layer consists of
$112\plh 112$ RGB image and with an output layer of 128 nodes
representing the face ``signature''.  For the choice of network
topology we follow the network-in-network (``NIN'') concept
of~\cite{LinCY13NIN} which was later refined in~\cite{Zeiler2015}.
The basic structure of a NIN consists of spatial convolutions followed
by the same number of $1\plh 1$ convolutions. A typical NIN block is
as follows: $C$ convolution kernels with $3\plh 3$ spatial resolution
which are typically strided followed by ReLU.  The block is then
passed through $C$ $1\plh 1$ convolutions and ReLU.
Table~\ref{fig:CNNmodel} presents the ConvNet architecture consisting
of $1.3M$ parameters while taking $41M$ FLOPs which is roughly
$30$msec runtime on a single Cortex A9 core.

\begin{figure}[t]
  \centering 
  \input{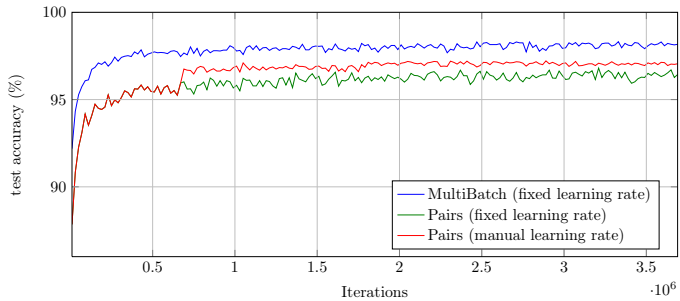}
  \caption{\small Multibatch convergence compared to using pairs.  
    Here we show LWF Accuracy as defined in~\cite{LFWTech} by a 
    10-fold cross validation.
    The speed advantage of multibatch is clear, even when we manually
    reduced the learning rate of SGD tenfold whenever it stopped (at
    $0.7m$ and $1.7m$ iterations) to help it converge faster.  This
    demonstrates how improved gradient estimation accuracy affects
    convergence.}
  \label{fig:convergence-graph}
\end{figure}

Training was done with multibatch, while setting $k=256$. In choosing
the $256$ examples of the mini-batch, we randomly picked $16$
individuals, and randomly picked $16$ images per individual. We
applied Eq.~\ref{eq:pair-loss} as the loss per-pair, while weighting
down the loss of ``not-same'' pairs so that FN and FP have the same
impact, since equal error rate is the standard performance measure on
LFW.  We used learning rate of $0.01$ throughout, and reduced it to
$0.001$ for the last epoch. For SGD we decreased the learning rate
whenever it stopped decreasing in order to help it further.
Figure~\ref{fig:convergence-graph} shows the test accuracy on the LFW
dataset throughout the iterations of the multibatch training method
compared to the standard pairwise SGD approach. One can clearly see
that even with manual tweaking of the learning rate the pairwise SGD
is significantly behind the multibatch algorithm. Both methods
``flatten'' around 1M iterations and further iterations do not help
the pairwise SGD to recover and catchup with multibatch.

\begin{table}
\centering
\tiny
\begin{tabular}{|c|c|c|c|c|c|c|r|}
  \hline
  Method                              & Training size & \# Parameters & FLOPs & Ensemble & Total FLOPs & LFW Accuracy  & Speed      \\
  \hline
  MultiBatch                    & 2.6M          & {\bf 1.3M}    & 41M   & 1        & {\bf 41M}   & {\bf 98.20}\% & $1\plh$    \\
  FaceNet~\cite{Schroff2015FaceNet}   & 260M          & 140M          & 1.6B  & 1        & 1.6B        & 99.63\%       & $40\plh$   \\
  DeepFace~\cite{Taigman2014deepface} & 4.4M          & 120M          & 3.8B  & 5        & 19.3B       & 97.35\%       & $4825\plh$ \\
  VGG~\cite{parkhi15deep}             & 2.6M          & 133M          & 11.3B & 30       & 340B        & 99.13\%       & $8500\plh$ \\
  \hline
\end{tabular}
\caption{Comparison between leading published work on face
  recognition sorted by runtime.}
\label{fig:comparison}
\end{table}

\begin{table}  
\centering
\tiny
\begin{tabular}{|l|l|r|r|}
  \hline
  Input Size          & Type                                        & FLOPs & Parameters \\
  \hline
  $50\plh 50\plh 3$   & 4 Convolutions, $5\plh 5$                   & 760K  & 304        \\
  $25\plh 25\plh 4$   & 12 Convolutions, $5\plh 5$                  & 758K  & 2K         \\
  $24\plh 24\plh 12$  & 12 Convolutions, $5\plh 5$                  & 3M    & 4K         \\
  $12\plh 12\plh 12$  & 12 Convolutions, $3\plh 3$                  & 189K  & 2K         \\
  $6\plh 6\plh 12$    & 4 Convolutions, $3\plh 3$                   & 16K   & 436        \\
  $6\plh 6\plh 4$     & Affine+ReLU, 256                            & 38K   & 38K        \\
  $256$               & Affine+ReLU, 64                             & 17K   & 17K        \\
  $64$                & Affine+ReLU, 140                            & 10K   & 10K        \\
  $140$               & Affine+ReLU, 64                             & 10K   & 10K        \\
  $64$                & Affine+ReLU, 128                            & 9K    & 9K         \\
  $128$               & Affine+ReLU, 64                             & 9K    & 9K         \\
  $64$                & Affine+ReLU, 4                              & 260   & 260        \\
  \hline
  $112\plh 112\plh 3$ & Warp                                        & 150k  & 0          \\
  \hline
  $112\plh 112\plh 3$ & NIN with 32 kernels of $5\plh 5$ (stride 2) & 8M    & 30k        \\
  $56\plh 56\plh 32$  & max-pool $2\plh 2$                          & 0     & 0          \\
  $28\plh 28\plh 32$  & NIN with 96 kernels of $3\plh 3$ (stride 2) & 8M    & 121k       \\
  $14\plh 14\plh 96$  & NIN with 128 kernels of $3\plh3$ (stride 2) & 9M    & 165k       \\
  $7\plh 7\plh 128$   & NIN with 128 kernels of $3\plh3$            & 9M    & 165k       \\
  $7\plh 7\plh 128$   & NIN with 128 kernels of $3\plh 3$           & 4M    & 165k       \\
  $7\plh 7\plh 128$   & 128 Convolutions, $3\plh 3$    (stride 2)   & 8M    & 148k       \\
  $4\plh 4\plh 128$   & Affine+ReLU 256                             & 525k  & 525k       \\
  $256$               & Affine+ReLU 256                             & 66k   & 66k        \\
  $256$               & Affine 128                                  & 33k   & 33k        \\
  $128$               & Loss                                        &       &            \\
  \hline
  Total:              &                                             & 41M & 1.3M       \\
  \hline
\end{tabular}
\caption{\small Our ConvNet architecture combines alignment and
  signature generation.  The version here achieves an accuracy of
  98.2\% on LFW and takes 30ms on a Cortex A9 core.  We also trained a
  much larger VGG-like model with $11B$ FLOPs which takes $6.6sec$
  runtime and achieves $98.8\%$ --- an accuracy which appears to be
  limited by the noise in our training data (see text).}
\label{fig:CNNmodel}
\end{table}

The ConvNet presented in Table~\ref{fig:CNNmodel} achieves a precision
of $98.2\%$ on the LFW benchmark with a runtime of 30ms on a single
Cortex A9 ARM core. This result is superior to the 2014 state of the
art achieved by \cite{Taigman2014deepface} and better than human
performance on this dataset.  State of the art today achieves $99\%$
and above but requires network ensembles at runtime, much larger
networks, considerably larger training sets (e.g., $260M$ images used
by \cite{Schroff2015FaceNet}), and elaborate data augmentations such
as multiple crops and resolutions.  We also trained a much larger,
VGG-like, model with $11B$ FLOPs which takes $6.6sec$ runtime and
achieves $98.8\%$ --- an accuracy which appears to be limited by the
noise in our training data.


%

\section{Summary}

The main technical contribution of this work centers around extracting
a much smaller variance of estimating the gradient of the loss
function for SGD for embedded metric learning. We have shown that our
multibatch method introduces a dramatic acceleration of the training
time required to learn an effective embedding for the task of face
recognition. Along the way we have also addressed the issue of why
practitioners tend to devise a classification network for face
recognition rather than training a metric embedding network by showing
that metric embedding is a harder problem to train --- thus
underscoring the importance of introducing the multibatch method.


{\small
\bibliographystyle{plain}
\bibliography{multibatch}

\begin{thebibliography}{10}

\bibitem{ghadimi2013stochastic}
Saeed Ghadimi and Guanghui Lan.
\newblock Stochastic first-and zeroth-order methods for nonconvex stochastic
  programming.
\newblock {\em SIAM Journal on Optimization}, 23(4):2341--2368, 2013.

\bibitem{LFWTech}
Gary~B. Huang, Manu Ramesh, Tamara Berg, and Erik Learned-Miller.
\newblock Labeled faces in the wild: A database for studying face recognition
  in unconstrained environments.
\newblock Technical Report 07-49, University of Massachusetts, Amherst, October
  2007.

\bibitem{JaderbergSZK15Spatial}
Max Jaderberg, Karen Simonyan, Andrew Zisserman, and Koray Kavukcuoglu.
\newblock Spatial transformer networks.
\newblock {\em NIPS}, pages 2017--2025, 2015.

\bibitem{johnson2013accelerating}
Rie Johnson and Tong Zhang.
\newblock Accelerating stochastic gradient descent using predictive variance
  reduction.
\newblock In {\em Advances in Neural Information Processing Systems}, pages
  315--323, 2013.

\bibitem{kazemi2014}
Vahid Kazemi and Josephine Sullivan.
\newblock One millisecond face alignment with an ensemble of regression trees.
\newblock In {\em Proceedings of the 2014 IEEE Conference on Computer Vision
  and Pattern Recognition}, 2014.

\bibitem{LinCY13NIN}
Min Lin, Qiang Chen, and Shuicheng Yan.
\newblock Network in network.
\newblock {\em CoRR}, abs/1312.4400, 2013.

\bibitem{moulines2011non}
Eric Moulines and Francis~R Bach.
\newblock Non-asymptotic analysis of stochastic approximation algorithms for
  machine learning.
\newblock In {\em Advances in Neural Information Processing Systems}, pages
  451--459, 2011.

\bibitem{needell2014stochastic}
Deanna Needell, Rachel Ward, and Nati Srebro.
\newblock Stochastic gradient descent, weighted sampling, and the randomized
  kaczmarz algorithm.
\newblock In {\em Advances in Neural Information Processing Systems}, pages
  1017--1025, 2014.

\bibitem{parkhi15deep}
O.~M. Parkhi, A.~Vedaldi, and A.~Zisserman.
\newblock Deep face recognition.
\newblock In {\em Proceedings of the British Machine Vision Conference (BMVC)},
  2015.

\bibitem{Schroff2015FaceNet}
Florian Schroff, Dmitry Kalenichenko, and James Philbin.
\newblock Facenet: {A} unified embedding for face recognition and clustering.
\newblock {\em CoRR}, abs/1503.03832, 2015.

\bibitem{shalev2016sdca}
Shai Shalev-Shwartz.
\newblock Sdca without duality, regularization, and individual convexity.
\newblock {\em arXiv preprint arXiv:1602.01582}, 2016.

\bibitem{shalev2004online}
Shai Shalev-Shwartz, Yoram Singer, and Andrew~Y Ng.
\newblock Online and batch learning of pseudo-metrics.
\newblock In {\em Proceedings of the twenty-first international conference on
  Machine learning}, page~94. ACM, 2004.

\bibitem{shalev2013stochastic}
Shai Shalev-Shwartz and Tong Zhang.
\newblock Stochastic dual coordinate ascent methods for regularized loss.
\newblock {\em The Journal of Machine Learning Research}, 14(1):567--599, 2013.

\bibitem{Taigman2014deepface}
Yaniv Taigman, Ming Yang, MarcAurelio Ranzato, and Lior Wolf.
\newblock Deepface: Closing the gap to human-level performance in face
  verification.
\newblock In {\em Conference on Computer Vision and Pattern Recognition
  (CVPR)}, 2014.

\bibitem{xing2003distance}
Eric~P Xing, Andrew~Y Ng, Michael~I Jordan, and Stuart Russell.
\newblock Distance metric learning with application to clustering with
  side-information.
\newblock {\em Advances in neural information processing systems}, 15:505--512,
  2003.

\bibitem{Zeiler2015}
Matthew~D. Zeiler and Rob Fergus.
\newblock Visualizing and understanding convolutional networks.
\newblock {\em ECCV 2014}, 2014.

\bibitem{zhao2014accelerating}
Peilin Zhao and Tong Zhang.
\newblock Accelerating minibatch stochastic gradient descent using stratified
  sampling.
\newblock {\em arXiv preprint arXiv:1405.3080}, 2014.

\end{thebibliography}
}
\newpage
\appendix
\section{Appendix: Proof of \thmref{thm:mainVarianceMultibatch}} \label{sec:proofMain}
We first show that the estimate is unbiased. Indeed, for every
$i \neq j$ we can rewrite $L(z)$ as $\E_\pi\ell_{\pi(i),\pi(j)}(z)$. Therefore,
\[
L(z) = \frac{1}{k^2-k} \sum_{i \neq j \in [k]} L(z) =  \frac{1}{k^2-k} \sum_{i \neq j \in [k]} \E_\pi
\ell_{\pi(i),\pi(j)}(z) = \E_\pi L_\pi(z) ~,
\]
which proves that the multibatch estimate is unbiased.

Next, we turn to analyze the variance of the multibatch estimate. let
$I \subset [k]^4$ be all the indices $i,j,s,t$ s.t.
$i\neq j, s \neq t$, and we partition $I$ to $I_1 \cup I_2 \cup I_3$,
where $I_1$ is the set where $i=s$ and $j=t$, $I_2$ is when all
indices are different, and $I_3$ is when $i=s$ and $j \neq t$ or
$i \neq s$ and $j=t$. Then:
\begin{align*}
&\E_\pi \| \nabla L_\pi(z) - \nabla L(z)\|^2 = \frac{1}{(k^2-k)^2} \E_\pi
\sum_{(i,j,s,t) \in I} (\nabla \ell_{\pi(i),\pi(j)}(z) - \nabla
L(z)) \cdot  (\nabla \ell_{\pi(s),\pi(t)}(z) - \nabla
L(z))  \\
&= \sum_{r=1}^d \frac{1}{(k^2-k)^2} \sum_{q=1}^3 \sum_{(i,j,s,t) \in I_q}\E_\pi
  (\nabla_r \ell_{\pi(i),\pi(j)}(z) - \nabla_r
L(z))\, (\nabla_r \ell_{\pi(s),\pi(t)}(z) - \nabla_r
L(z))
\end{align*}
For every $r$, denote by $A^{(r)}$ the matrix with
$A^{(r)}_{i,j} = \nabla_r \ell_{i,j}(z) - \nabla_r
L(z)$. Observe that for every $r$, $\E_{i \neq j} A^{(r)}_{i,j} = 0$,
and that
\[
\sum_r \E_{i \neq j} (A^{(r)}_{i,j})^2 = \E_{i \neq j} \| \nabla \ell_{i,j}(z)
- \nabla L(z)\|^2.
\]

Therefore,
\[
\E_\pi \| \nabla L_\pi(z) - \nabla L(z)\|^2 = \sum_{r=1}^d
\frac{1}{(k^2-k)^2} \sum_{q=1}^3 \sum_{(i,j,s,t) \in I_q}\E_\pi A^{(r)}_{\pi(i),\pi(j)} A^{(r)}_{\pi(s),\pi(t)}
\]
Let us momentarily fix $r$ and omit the superscript from $A^{(r)}$.
We consider the value of $\E_\pi
  A_{\pi(i),\pi(j)} A_{\pi(s),\pi(t)}  $ according to the value of $q$.
\begin{itemize}
\item For $q=1$: we obtain $\E_\pi A_{\pi(i),\pi(j)}^2$ which is the variance
of the random variable $\nabla_r \ell_{i,j}(z) - \nabla_r L(z)$.

\item For $q=2$: When we fix $i,j,s,t$ which are all different, and
  take expectation over $\pi$, then all products of off-diagonal
  elements of $A$ appear the same number of times in $\E_\pi
  A_{\pi(i),\pi(j)} A_{\pi(s),\pi(t)} $. Therefore, this quantity is
  proportional to $\sum_{p \neq r} v_p v_r$, where $v$ is the vector
  of all non-diagonal entries of $A$. Since $\sum_p v_p = 0$,
  we obtain (using \lemref{lem:worepSamp}) that $\sum_{p \neq r} v_p
  v_r \le 0$, which means that the entire sum for this case is
  non-positive.

\item For $q=3$: Let us consider the case when $i=s$ and $j \neq t$,
  and the derivation for the case when $i \neq s$ and $j=t$ is
  analogous. The expression we obtain is $\E_\pi A_{\pi(i),\pi(j)}
  A_{\pi(i),\pi(t)}$.  This is like first sampling a row and then
  sampling, without replacement, two indices from the row (while not
  allowing to take the diagonal element).  So, we can rewrite the
  expression as:
  \begin{equation}
    \begin{split}
      \E_\pi
      A_{\pi(i),\pi(j)} A_{\pi(s),\pi(t)}  = &
      \E_{i \sim [m]} \E_{j,t \in [m] \setminus \{i\} : j \neq t} A_{i,j}
      A_{i,t} \\
      &~\le~ \E_{i \sim [m]} \left( \E_{j \neq i} A_{i,j} \right)^2 = \E_{i \sim [m]} (\bar{A}_i)^2 ~,
    \end{split}
  \end{equation}
where we denote $\bar{A}_i = \E_{j \neq i} A_{i,j}$ and in the
inequality we used again \lemref{lem:worepSamp}.
\end{itemize}
Finally, the bound on the variance follows by observing that the
number of summands in $I_1$ is $k^2-k$ and the number of summands in
$I_3$ is $O(k^3)$. This concludes our proof.

\begin{lemma} \label{lem:worepSamp}
Let $v \in \reals^n$ be any vector. Then,
\[
\E_{s \neq t} [v_s v_t] \le (\E_i[v_i])^2
\]
In particular, if $\E_i[v_i]=0$ then $\sum_{s \neq t} v_s v_t \le 0$.
\end{lemma}
\begin{proof}
For simplicity, we use $\E[v]$ for $\E_i[v_i]$ and $\E[v^2]$ for $\E_i
[v_i^2]$. Then:
\begin{align*}
\E_{s \neq t} v_s v_t &= \frac{1}{n^2-n} \sum_{s=1}^n \sum_{t=1}^n v_s v_t -
\frac{1}{n^2-n}  \sum_{s=1}^n v_s^2 \\
&=  \frac{1}{n^2-n} \sum_{s=1}^n v_s \sum_{t=1}^n v_t -
\frac{1}{n^2-n} \sum_{s=1}^n v_s^2 \\
&= \frac{n^2}{n^2-n} \E[v]^2 - \frac{n}{n^2-n}\E[v^2] \\
&= \frac{n}{n^2-n} (\E[v]^2-\E[v^2]) + \frac{n^2-n}{n^2-n} \E[v]^2 \\
&\le 0 + \E[v]^2
\end{align*}
\end{proof}

\end{document}